\documentclass[12pt]{article}
\usepackage[colorlinks,urlcolor=blue,linkcolor=blue,citecolor=blue]{hyperref}

\usepackage{amsmath, amsthm, amsfonts, amssymb}
\usepackage{geometry}
\usepackage{hyperref}
\usepackage{algorithm}
\usepackage{algpseudocode}
\usepackage{caption}
\usepackage{subcaption}
\usepackage{graphicx}
\newtheorem{lemma}{Lemma}
\newtheorem{corollary}{Corollary}

\graphicspath{
  {Python/}
  {images/}
}

\usepackage{booktabs}

\usepackage{color,array}

\usepackage{graphicx}

\theoremstyle{plain}
\newtheorem{theorem}{Theorem}[section]

\theoremstyle{definition}

\theoremstyle{remark}

\begin{document}

\title{Theoretical Convergence of SMOTE-Generated Samples}

\author{
Firuz Kamalov$^{1}$ \and
Hana Sulieman$^{2}$ \and
Witold Pedrycz$^{3}$\\[0.5em]
$^{1}$Department of Computational Sciences, Canadian University Dubai\\
$^{2}$Department of Mathematics and Statistics, American University of Sharjah\\
$^{3}$Electrical \& Computer Engineering Department, University of Alberta
}

\maketitle

\begin{abstract}
Imbalanced data affects a wide range of machine learning applications, from healthcare to network security. As SMOTE is one of the most popular approaches to addressing this issue, it is imperative to validate it not only empirically but also theoretically. 
In this paper, we provide a rigorous theoretical analysis of SMOTE’s convergence properties. Concretely, we prove that  the synthetic random variable \( Z \) converges in probability to the underlying random variable \( X \). We further prove a stronger convergence in mean when \( X \) is compact. Finally, we show that lower values of the nearest neighbor rank lead to faster convergence offering actionable guidance to practitioners. The theoretical results are supported by numerical experiments using both real-life and synthetic data.
Our  work provides a foundational understanding that enhances data augmentation techniques beyond imbalanced data scenarios.

\end{abstract}


\section{Introduction}

Imbalanced datasets are a pervasive challenge in machine learning and statistical analysis, where one class (the minority class) is significantly underrepresented compared to others (the majority class). The imbalance often leads to biased predictive models that perform poorly on the minority class~\cite{Thabtah}, which may represent critical cases such as fault detection \cite{Kuang}, medical diagnoses \cite{Khushi}, network security \cite{Ding}, or computer vision \cite{Li2021}. To mitigate this issue, various data augmentation techniques have been developed to generate synthetic samples of the minority class, thereby balancing the dataset and improving model performance.

One of the most widely used augmentation methods is the Synthetic Minority Over-sampling Technique (SMOTE), introduced by Chawla et al.~\cite{Chawla}. SMOTE generates new synthetic samples by interpolating between existing minority class instances and their nearest neighbors. Specifically, for a given minority class sample, synthetic points are created along the line segments joining it with its \( k \)-nearest neighbors. This approach has been empirically successful and has been integrated into numerous applications and extended in various ways~\cite{Dablain, Wang}.

Despite its widespread adoption and a plethora of empirical studies that analyze its properties, the theoretical aspects of SMOTE remain relatively unexplored. In particular, there is limited understanding of how the synthetic samples generated by SMOTE relate to the original data distribution as the sample size increases. 
The question of convergence is central to understanding the effectiveness of the SMOTE algorithm.
The gap in the literature is significant because a theoretical foundation would provide insights into the algorithm's effectiveness, inform parameter selection, and guide the development of improved augmentation methods.

In this paper, we fill the gap in analytical framework by providing a theoretical analysis of the convergence of SMOTE-generated samples to the original data distribution. Our study considers a continuous random variable \( X \) and examines the behavior of the synthetic random variable \( Z \), which is generated from an independent and identically distributed (i.i.d.) sample of \( X \).

Our main contributions are as follows:

\begin{enumerate}
    \item \textbf{Convergence in Probability:} We prove that the synthetic random variable \( Z \) converges in probability to the original random variable \( X \) as the sample size \( n \) approaches infinity (Theorem~\ref{convProb}). This result establishes that, with sufficiently large samples, the distribution of SMOTE-generated samples approximates the original data distribution.

    \item \textbf{Impact of Nearest Neighbor Rank \( k \):} We analyze the effect of the choice of \( k \) on the convergence rate of \( Z \) to \( X \). Our analysis reveals that selecting lower values of \( k \) leads to faster convergence. Specifically, we demonstrate through Equation~\ref{compk} that the expected distance between a sample point and its \( k \)-th nearest neighbor increases with \( k \). Consequently, synthetic samples generated using higher values of \( k \) are more likely to be farther from the original data points, potentially slowing down convergence.

    \item \textbf{Empirical Validation:} We support our theoretical findings with simulation studies using uniform, Gaussian, and exponential distributions. The simulations illustrate the convergence behavior of \( Z \) for different values of \( k \) and sample sizes \( n \) (Figures~\ref{fig:uniform}--\ref{fig:gaussian}). We observe that lower values of \( k \) result in faster convergence of \( Z \) to the distribution of \( X \), as measured by the Kullback-Leibler (KL) divergence (Figure~\ref{fig:kl}).

    \item \textbf{Convergence in Mean:} Under the additional assumption that \( X \) has compact support, we establish that \( Z \) converges to \( X \) in mean as \( n \rightarrow \infty \) (Theorem~\ref{thm:convmean}). This stronger form of convergence provides a deeper understanding of the relationship between the synthetic and original data distributions.
\end{enumerate}

Our results have practical implications for the application of SMOTE in data augmentation. They suggest that using the nearest neighbor (\( k = 1 \)) in the SMOTE algorithm is preferable for achieving faster convergence to the original distribution. This insight can guide practitioners in selecting appropriate parameters for SMOTE and potentially improve the performance of models trained on augmented data.

The remainder of the paper is organized as follows. In Section~\ref{sec:main}, we present the main theoretical results, including proofs of convergence in probability and mean. Section~\ref{sec:num} details the simulation studies that corroborate our theoretical findings and explore the effects of different choices of \( k \) and sample sizes \( n \). Finally, in Section~\ref{sec:conclusion}, we discuss the implications of our work and suggest directions for future research.

We believe that this work contributes significantly to the theoretical understanding of SMOTE and opens avenues for further exploration of data augmentation techniques in machine learning.

\begin{figure*}[!t]
\centering
\begin{equation}
\label{eq:elreedy}
\begin{split}
    f_Z(z)&= (N-K) \binom{N-1}{K} \int_x f_X(x)
 \int_{r=|| z-x||}^\infty f_X \left(x+ \frac{(z-x)r}{||z-x||}\right)
 \left( \frac{r^{d-2}}{||z-x||^{d-1}} \right)  \\
&\times B\left(1-I_{B(x,r);N-K-1,K} \right)dr dx.  
\end{split}
\end{equation}
\end{figure*}
\section{Background}
\label{sec:lit}

Handling imbalanced data remains a dynamic area of research. Although numerous sampling algorithms have been developed \cite{Xie, Liu2023}, the SMOTE algorithm is arguably the most widely recognized method in the literature. Initially introduced by Chawla et al. in 2002 \cite{Chawla}, SMOTE has been successfully applied across a variety of applications.
There are numerous extensions of SMOTE that aim to enhance the original algorithm’s sampling performance. For example, the Adaptive Synthetic (ADASYN) algorithm \cite{He2008} is similar to SMOTE but varies the number of generated samples based on the estimated local distribution of the class being oversampled. Another widely used variant, Borderline SMOTE \cite{Han}, identifies borderline samples and utilizes them to create new synthetic instances. Additionally, SVM-SMOTE was developed in \cite{Nguyen} by incorporating support vector machines into the SMOTE framework. More recently, several other extensions have been proposed, including Center Point SMOTE, Inner and Outer SMOTE \cite{Bao}, DeepSMOTE~\cite{Dablain}, and Deep Attention SMOTE \cite{Liu}, among others, further expanding the applicability and effectiveness of SMOTE in various contexts.

Despite its popularity and wide-spread use, there exist very few studies that analyze theoretical properties of SMOTE. In early attempt to understand theoretical aspects of sampling strategies, King and Zeng \cite{King} studied the random undersampling strategy in the context of the logistic regression.
In \cite{Elreedy2019},  Elreedy and Atiya derived the expectation and covariance matrix of the SMOTE generated patterns. More recently, Elreedy et al. developed the theoretical distribution of $Z$ \cite{Elreedy2023} described in Equation~\ref{eq:elreedy}. While enlightening, the expression in Equation~\ref{eq:elreedy} is not practical as it is expressed in the form of a double integral. 
In \cite{Sakho}, Sakho et al. proposed an alternative derivation of Equation~\ref{eq:elreedy} using random variables instead of geometrical arguments. Sakho et al. also studied the convergence of $Z$ to $X$ and showed that if $X$ has a bounded support, then $Z_{K,n}\mid X_c = x_c\to x_c$ in probability (Theorem 3.3~\cite{Sakho}). In \cite{Kamalov2024}, the authors improve the convergence result of \cite{Sakho} by showing that $Z$ converges to $X$ conditionally in mean, when $X$ is left-bounded.

Our research advances existing literature in two significant ways. First, we demonstrate that \( Z \) converges to \( X \) in probability without requiring any assumptions about the boundedness of the support. Second, we show that when the support is bounded, \( Z \) converges to \( X \) in mean, which is a stronger form of convergence than convergence in probability.

\section{Main Results}
\label{sec:main}
Let \( X \) be a real-valued random variable with cumulative distribution function \( F \) and probability density function \( f \). Consider an independent and identically distributed (i.i.d.) sample \( X_1, X_2, \dots, X_n \) drawn from \( X \). Let $Z$ be the random variable generated via the SMOTE-$k$ procedure, as described in Algorithm 1. We will show that $Z$ converges to $X$ in probability as $n\rightarrow \infty$.

\begin{algorithm}
\caption{SMOTE-$k$}
\label{alg:smote1}
\begin{algorithmic}[1]
\State \textbf{Input:} A sample $(X_1, X_2, \dots, X_n)$ and neighbor rank $1\leq k\leq n-1$
\State \textbf{Output:} A synthetic sample $Z$
\State Randomly choose an instance $X_i$ from the sample
\State Find the $k$-th nearest neighbor of $X_i$, denoted $X_{i,(k)}$
\State Generate a random number $\lambda \sim U(0, 1)$
\State Create a synthetic point $Z= X_i + \lambda (X_{i,(k)} - X_i)$
\State \textbf{Return} $Z$
\end{algorithmic}
\end{algorithm}

We note that there exists a more general version of the SMOTE-$k$ algorithm which will be discussed in Section~\ref{sec:gensmote}. Specifically, all the results established for the SMOTE-$k$ algorithm are also applicable to this broader algorithm.

\begin{theorem}
\label{convProb}
Let \( X \) be a continuous random variable. Let $Z$ be the random variable generated via the SMOTE-k procedure from an  i.i.d. sample  \( X_1, X_2, \dots, X_n \) drawn from \( X \). Then, $Z$ converges to $X$ in probability as $n\rightarrow \infty$.
\end{theorem}

\begin{proof}
Let \( F \) be the continuous cumulative distribution function and \( f \) be the corresponding probability density function of $X$. Since  \( X_1, X_2, \dots, X_n \)  is an i.i.d. sample, we can assume without the loss of generality that $Z$ is generated between $X_1$ and its $k$-th nearest neighbor $X_{1,(k)}$. We will show that $Z$ converges to $X_1$ in probability, i.e., 
\[
\quad \lim_{n \to \infty} P\left( \left| Z - X_1 \right| > \epsilon \right) = 0, \,\forall \epsilon > 0.
\]
Since $|Z - X_1| \leq |X_{1,(k)} - X_1|$ then 
\[
P\left( \left| Z - X_1 \right| > \epsilon \right) \leq P\left( \left| X_{1,(k)} - X_1 \right| > \epsilon \right).
\]
Thus, it suffices to show that 
\[
\quad \lim_{n \to \infty} P\left( \left| X_{1,(k)} - X_1 \right| > \epsilon \right) = 0,  \,\forall \epsilon > 0.
\]
For each \( i \neq 1 \), define the distance:
\[
D_i = |X_i - X_1|.
\]
Arrange these distances in ascending order:
\[
D_{(1)} \leq D_{(2)} \leq \dots \leq D_{(n - 1)},
\]
so that \( D_{(k)} = |X_{1,(k)} - X_1| \) is the \( k \)-th smallest distance from \( X_1 \) to the other sample points.
Our objective is to estimate \( P(D_{(k)} \geq \varepsilon) \).

Let \( Y \) denote the number of sample points among \( X_2, X_3, \dots, X_n \) that are within \( \varepsilon \) of \( X_1 \):
\[
Y = \sum_{i=2}^{n} \mathbf{1}_{\{ D_i \leq \varepsilon \}},
\]
where \( \mathbf{1}_{\{ D_i \leq \varepsilon \}} \) is the indicator function that equals 1 if \( D_i \leq \varepsilon \) and 0 otherwise.
Given \( X_1 = x_1 \), each \( D_i \) is independent of the others, and the events \( \{ D_i \leq \varepsilon \} \) are independent Bernoulli trials with success probability:
\[
p(x_1) = P(D_i \leq \varepsilon \mid X_1 = x_1) = \int_{x_1 - \varepsilon}^{x_1 + \varepsilon} f(x) \, dx.
\]
Thus, conditional on \( X_1 = x_1 \), \( Y \) follows a binomial distribution:
\[
Y \mid X_1 = x_1 \sim \text{Binomial}(n - 1, p(x_1)).
\]
We can express the conditional probability \(P(D_{(k)} \geq \varepsilon \mid X_1)\) as:
 
\begin{equation}
\label{distk}
P(D_{(k)} \geq \varepsilon \mid X_1) =  P(Y \leq k - 1 \mid X_1).
\end{equation}
To find the unconditional probability, we take the expectation over \( X_1 \):
\begin{equation}
\label{uncond}
P(D_{(k)} \geq \varepsilon) = E_{X_1}\left[ P(Y \leq k - 1 \mid X_1) \right].
\end{equation}
Our goal is to show that this expectation tends to zero as \( n \to \infty \).
To this end, note that for any \( \delta > 0 \), since \( f \) is integrable, there exists a compact set \( S \subset \mathbb{R} \) such that:
\begin{equation}
\label{bound}
\int_{S^c} f(x_1) \, dx_1 < \delta.
\end{equation}
Since \( S \) is compact and \( p(x_1) \) is strictly positive,  there exist  \( p_{\min} \gneq 0\) such that 
\begin{equation}
\label{comp}
 p(x_1) \geq p_{\min}, \,\forall x_1 \in S.
\end{equation}
We use  \( S \) to split \( P(D_{(k)} \geq \varepsilon) \) into two parts to be handled separately:
\begin{equation}
\label{break}
P(D_{(k)} \geq \varepsilon) = P(D_{(k)} \geq \varepsilon, X_1 \in S) + P(D_{(k)} \geq \varepsilon, X_1 \in S^c).
\end{equation}

\subsubsection*{Case 1:  \( P(D_{(k)} \geq \varepsilon, X_1 \in S) \)}
Using Equation \ref{uncond} and \ref{comp}, we get:
\[
\begin{aligned}
&P(D_{(k)} \geq \varepsilon, X_1 \in S) = E_{X_1}\left[ P(Y \leq k - 1 \mid X_1  \in S) \right]\\
&= \int_{S}  P(Y \leq k - 1) f(x_1) \, dx_1\\
&=\sum_{j=0}^{k-1} \int_{S} \binom{n-1}{j} [p(x_1)]^j [1 - p(x_1)]^{n-1-j} f(x_1) \, dx_1\\
&\leq\sum_{j=0}^{k-1} (n-1)^{j} [1 -  p_{\min}]^{n-1-j} \cdot \int_{S} [p(x_1)]^j  f(x_1) \, dx_1\\
&\leq \sum_{j=0}^{k-1} (n-1)^{j} [1 -  p_{\min}]^{n-1-j}
\end{aligned}
\]

Since $(1 -  p_{\min} )<1$, then $(1 -  p_{\min})^{n-1-j}$ decreases exponentially,  as $n\to \infty$. On the other hand, $(n-1)^{j}$ increases in polynomial order. It follows that
\[
\lim_{n\to \infty} \sum_{j=0}^{k-1} (n-1)^{j} [1 -  f_{\min}]^{n-1-j}=0.
\]
Hence, 
\(
\lim_{n\to \infty}P(D_{(k)} \geq \varepsilon, X_1 \in S)=0.
\)

\subsubsection*{Case 2: \( P(D_{(k)} \geq \varepsilon, X_1 \in S^c) \)}
Using Equation \ref{bound}, we have:
\begin{align*}
  P(D_{(k)} \geq \varepsilon, X_1 \in S^c) &\leq P(X_1 \in S^c)\\
  &= \int_{S^c} f(x_1) \, dx_1 < \delta.  
\end{align*}

It follows from Equation \ref{break} that
\(
\lim_{n\to \infty}P(D_{(k)} \geq \varepsilon) < \delta.
\)
Since \( \delta > 0 \) is arbitrary we have:
\[
\lim_{n \to \infty} P(D_{(k)} \geq \varepsilon) = 0.
\]
Since \( |Z - X_1| \leq D_{(k)} \), then:
\[
P(|Z - X_1| > \varepsilon) \leq P(D_{(k)} \geq \varepsilon) \to 0 \quad \text{as} \quad n \to \infty.
\]
Therefore, \( Z \) converges in probability to \( X_1 \).
Given that \( X_1 \) is distributed as \( X \), and convergence in probability is preserved under convergence to a random variable, we conclude that:
\[
\lim_{n \to \infty} P(|Z - X| > \varepsilon) = 0.
\]
\end{proof}

To validate Theorem~\ref{convProb} on real-world data, we analyzed carbon monoxide readings from the UCI Air Quality dataset \cite{UCIvito}, treating the full series as the ground truth $X$. We generated synthetic samples $Z$ from subsets of varying sizes $n$ using neighbor ranks $k=1$ and $k=5$, measuring the Kolmogorov-Smirnov (KS) distance to $X$. As shown in Figure~\ref{fig:exp1CO}, the error decreases monotonically as $n$ increases, confirming convergence in probability. Moreover, $k=1$ yields consistently lower KS statistics than $k=5$ at smaller sample sizes, empirically supporting our theoretical conclusion that lower neighbor ranks lead to faster convergence.

\begin{figure*}[!h]
         \centering
         \includegraphics[width=\textwidth, trim=0 0 0 1.3cm, clip]{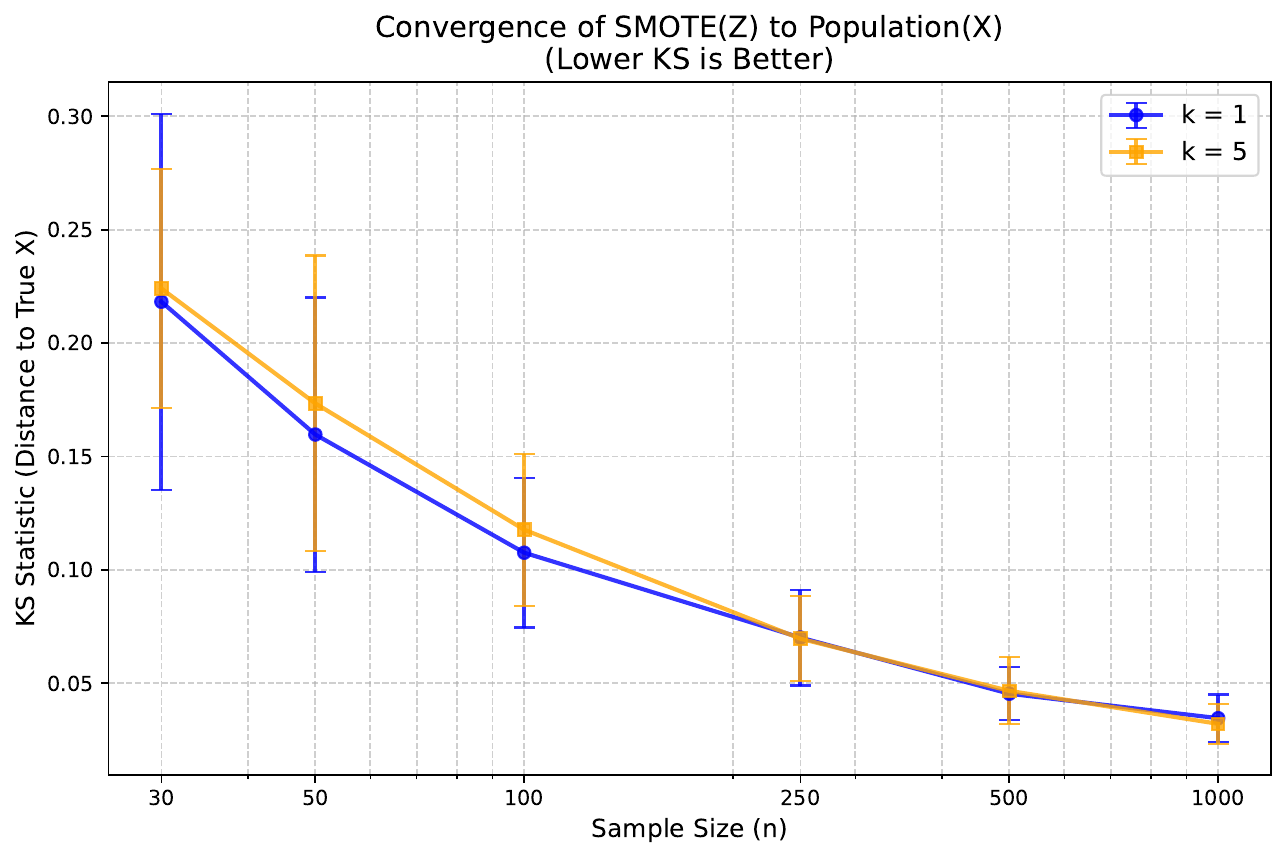}
        \caption{Convergence of SMOTE-generated samples $Z$ to the true population $X$ (UCI Air Quality CO data) as measured by the KS statistic. The downward trend confirms convergence in probability (Theorem~\ref{convProb}). Lower values of $k$ result in faster convergence.}
        \label{fig:exp1CO}
\end{figure*}

In the next theorem, we consider a continuous random variable with compact support. We show that the expected difference between any two consecutive order statistics tends to zero uniformly  as the sample size increases. As a consequence (Theorem \ref{thm:convmean}), we will obtain that $Z$ converges to $X$ in mean. Before proceeding further, we need the following auxiliary lemma whose proof is trivial.

\begin{lemma}
\label{conditionalExpect}
Let \( X \) and \( Y \) be random variables. The expectation of \( X \) can be expressed in terms of conditional expectations as
\[
E[X] = \sum_y E[X \mid Y = y] \cdot f_Y(y).
\]
\end{lemma}

\begin{theorem}
\label{meanConverge}
Let \( X \) be a continuous random variable with finite support \([a, b]\).  Let \( X_1, X_2, \dots, X_n \) be independent and identically distributed (i.i.d.) samples from \( X \), and let \( X_{(1)} \leq X_{(2)} \leq \dots \leq X_{(n)} \) denote their order statistics.
Then $\lim_{n\to\infty}E[X_{(k+1)} - X_{(k)}] =0$ uniformly.
\end{theorem}

\begin{proof}
Assume without the loss of generality that the cumulative distribution function \( F \)  is continuous and strictly increasing on \([a, b]\). Our goal is to show that for any \( \varepsilon > 0 \), there exists an \( N \in \mathbb{N} \) such that for all \( n > N \) and all \( 1 \leq k \leq n - 1 \),
\[
E[X_{(k+1)} - X_{(k)}] < \varepsilon.
\]
Since \( X \) has support \([a, b]\), then \( F \) is a bijection from \([a, b]\) to \([0, 1]\) and its inverse function \( F^{-1} \) exists and is continuous and strictly increasing on \([0, 1]\). Moreover, both \( F \) and \( F^{-1} \) are uniformly continuous as they are continuous functions on compact intervals.
Define
\[
U_i = F(X_i), \quad \text{for } i = 1, 2, \dots, n.
\]
Since \( X_i \) are i.i.d. samples from \( X \), the \( U_i \) are i.i.d. samples from the uniform distribution on \([0, 1]\). Let \( U_{(1)} \leq U_{(2)} \leq \dots \leq U_{(n)} \) denote the order statistics of the \( U_i \). 
The expected difference between consecutive order statistics for the uniform distribution on \([0, 1]\) is known \cite{Kamalov2024}:
\begin{equation}
\label{expectedDiff}
E[U_{(k+1)} - U_{(k)}] = \frac{1}{n + 1}, \quad \text{for } k = 1, 2, \dots, n - 1.
\end{equation}
Since \( F^{-1} \) is uniformly continuous, for any \( \varepsilon > 0 \), there exists a \( \delta > 0 \) such that for all \( u, v \in [0, 1] \),
\[
|u - v| < \delta \implies |F^{-1}(u) - F^{-1}(v)| < \varepsilon.
\]
We use Lemma \ref{conditionalExpect} to decompose \( E[X_{(k+1)} - X_{(k)}] \) based on whether \( U_{(k+1)} - U_{(k)} \) is less than \( \delta \):
\begin{align*}
E[X_{(k+1)} - X_{(k)}] &= E\left[ (X_{(k+1)} - X_{(k)}) \cdot \mathbb{I}\left( U_{(k+1)} - U_{(k)} < \delta \right) \right] \\
&\quad + E\left[ (X_{(k+1)} - X_{(k)}) \cdot \mathbb{I}\left( U_{(k+1)} - U_{(k)} \geq \delta \right) \right],
\end{align*}
where \( \mathbb{I}(\cdot) \) is the indicator function.

\subsection*{Case \( U_{(k+1)} - U_{(k)} < \delta \):}

When \( U_{(k+1)} - U_{(k)} < \delta \), the uniform continuity of \( F^{-1} \) implies
\[
|X_{(k+1)} - X_{(k)}| = |F^{-1}(U_{(k+1)}) - F^{-1}(U_{(k)})| < \varepsilon.
\]
Therefore,
\[
(X_{(k+1)} - X_{(k)}) \cdot \mathbb{I}\left( U_{(k+1)} - U_{(k)} < \delta \right) \leq \varepsilon \cdot \mathbb{I}\left( U_{(k+1)} - U_{(k)} < \delta \right).
\]
Taking expectations,
\begin{align*}
    &E\left[ (X_{(k+1)} - X_{(k)}) \cdot \mathbb{I}\left( U_{(k+1)} - U_{(k)} < \delta \right) \right]\\
    &\leq \varepsilon \cdot P\left( U_{(k+1)} - U_{(k)} < \delta \right)\\
    &\leq \varepsilon.
\end{align*}

\subsection*{Case \( U_{(k+1)} - U_{(k)} \geq \delta \):}

Since \( X_{(k)} \) and \( X_{(k+1)} \) are in \([a, b]\), the maximum possible difference is
\[
0 \leq X_{(k+1)} - X_{(k)} \leq b - a.
\]
Thus,
\begin{align*}
    &E\left[ (X_{(k+1)} - X_{(k)}) \cdot \mathbb{I}\left( U_{(k+1)} - U_{(k)} \geq \delta \right) \right]\\
    &\leq (b - a) \cdot P\left( U_{(k+1)} - U_{(k)} \geq \delta \right).
\end{align*}

To bound \( P\left( U_{(k+1)} - U_{(k)} \geq \delta \right) \), we apply Markov's inequality to the random variable \( U_{(k+1)} - U_{(k)} \). Recall that Markov's inequality states that for any non-negative random variable \( X \) and any \( t > 0 \),
\[
P(X \geq t) \leq \frac{E[X]}{t}.
\]
Applying this to \( U_{(k+1)} - U_{(k)} \) with \( t = \delta \),
\[
P\left( U_{(k+1)} - U_{(k)} \geq \delta \right) \leq \frac{E[U_{(k+1)} - U_{(k)}]}{\delta}.
\]
It follows  from Equation \ref{expectedDiff} that 
\[
P\left( U_{(k+1)} - U_{(k)} \geq \delta \right) \leq \frac{1}{n \delta}.
\]
Let \( n \) be large enough such that
\[
\frac{1}{n \delta} \leq \frac{\varepsilon}{2(b - a)}.
\]
Then
\[
(b - a) \cdot P\left( U_{(k+1)} - U_{(k)} \geq \delta \right) \leq \frac{\varepsilon}{2}. 
\]
It follows that 
\begin{align*}
&E[X_{(k+1)} - X_{(k)}]\\
&= E\left[ (X_{(k+1)} - X_{(k)}) \cdot \mathbb{I}\left( U_{(k+1)} -
 U_{(k)} < \delta \right) \right] \\
&\quad + E\left[ (X_{(k+1)} - X_{(k)}) \cdot \mathbb{I}\left( U_{(k+1)} - U_{(k)} \geq \delta \right) \right] \\
&\leq \varepsilon + (b - a) \cdot \frac{\varepsilon}{2(b - a)} = \varepsilon + \frac{\varepsilon}{2} = \frac{3\varepsilon}{2},
\end{align*}
for all \( n > N \) and all \( 1 \leq k \leq n - 1 \), as desired.
\end{proof}

We can generalize Theorem \ref{meanConverge} to the expected difference between any order statistics.

\begin{corollary}
\label{kConv}
In the same scenario as in Theorem \ref{meanConverge}, let $1\leq m \leq n-1$.
Then $\lim_{n\to\infty}E[X_{(k+m)} - X_{(k)}] =0$ uniformly, for all $1\leq k \leq n-m$.
\end{corollary}

\begin{proof}
Let $\varepsilon >0$ be given. By Theorem \ref{meanConverge}, there exists an \( N \in \mathbb{N} \) such that for all \( n > N \) and all \( 1 \leq k \leq n - 1 \),
\[
E[X_{(k+1)} - X_{(k)}] < \frac{\varepsilon}{m}.
\]
It follows that 
\begin{align*}
E[X_{(k+m)} - X_{(k)}] &=\sum_{i=k}^{m-1} E[X_{(i+1)} - X_{(i)}]\\
&\leq m\cdot  \frac{\varepsilon}{m} = \varepsilon.
\end{align*}
\end{proof}

In Theorem \ref{convProb}, we showed that $X$ converges to $Z$ in probability. Now, under additional condition of bounded support of $X$, we obtain a stronger result that $X$ converges to $Z$ in mean.

\begin{theorem}
\label{thm:convmean}
Let \( X \) be a continuous random variable with compact support. Let $Z$ be the random variable generated via the SMOTE-k procedure from an  i.i.d. sample  \( X_1, X_2, \dots, X_n \) drawn from \( X \). Then, $Z$ converges to $X$ in mean as $n\rightarrow \infty$.
\end{theorem}

\begin{proof}
Since  \( X_1, X_2, \dots, X_n \)  is an i.i.d. sample, we can assume without the loss of generality that $Z$ is generated between $X_1$ and its $m$-th nearest neighbor $X_{1,(m)}$. We will show that $Z$ converges to $X_1$ in mean, i.e., 
\[
\lim_{n\to\infty}E\left[ \left|Z - X_{1}\right| \right] =0.
\]
Let $\varepsilon >0$ be given. By Corollary \ref{kConv}, there exists an \( N \in \mathbb{N} \) such that for all \( n > N \) and all $1\leq k \leq n-m$,
\[
E[X_{(k+m)} - X_{(k)}] < \varepsilon.
\]
Then,
\begin{align*}
E[|Z - X_{1}|] &\leq E[|X_{1,(m)} - X_{1}|]\\
&\leq \sup_{1\leq k \leq n-m} E[X_{(k+m)} - X_{(k)}] \\
&\leq \varepsilon
\end{align*}
which completes the proof.
\end{proof}

To verify Theorem~\ref{thm:convmean}, we analyzed the "median income" feature from the California Housing dataset \cite{UCIpace}. The data is normalized to $[0, 1]$ to satisfy the assumption of compact support. We calculated the Wasserstein distance as suitable proxy for convergence in mean between the synthetic samples $Z$ and the true population $X$. As illustrated in Figure~\ref{fig:exp2_cali}, the distance decreases monotonically as the sample size $n$ increases. It empirically confirms that the synthetic distribution approaches the original distribution in mean.

\begin{figure*}[!h]
         \centering
         \includegraphics[width=\textwidth, trim=0 0 0 1.28cm, clip]{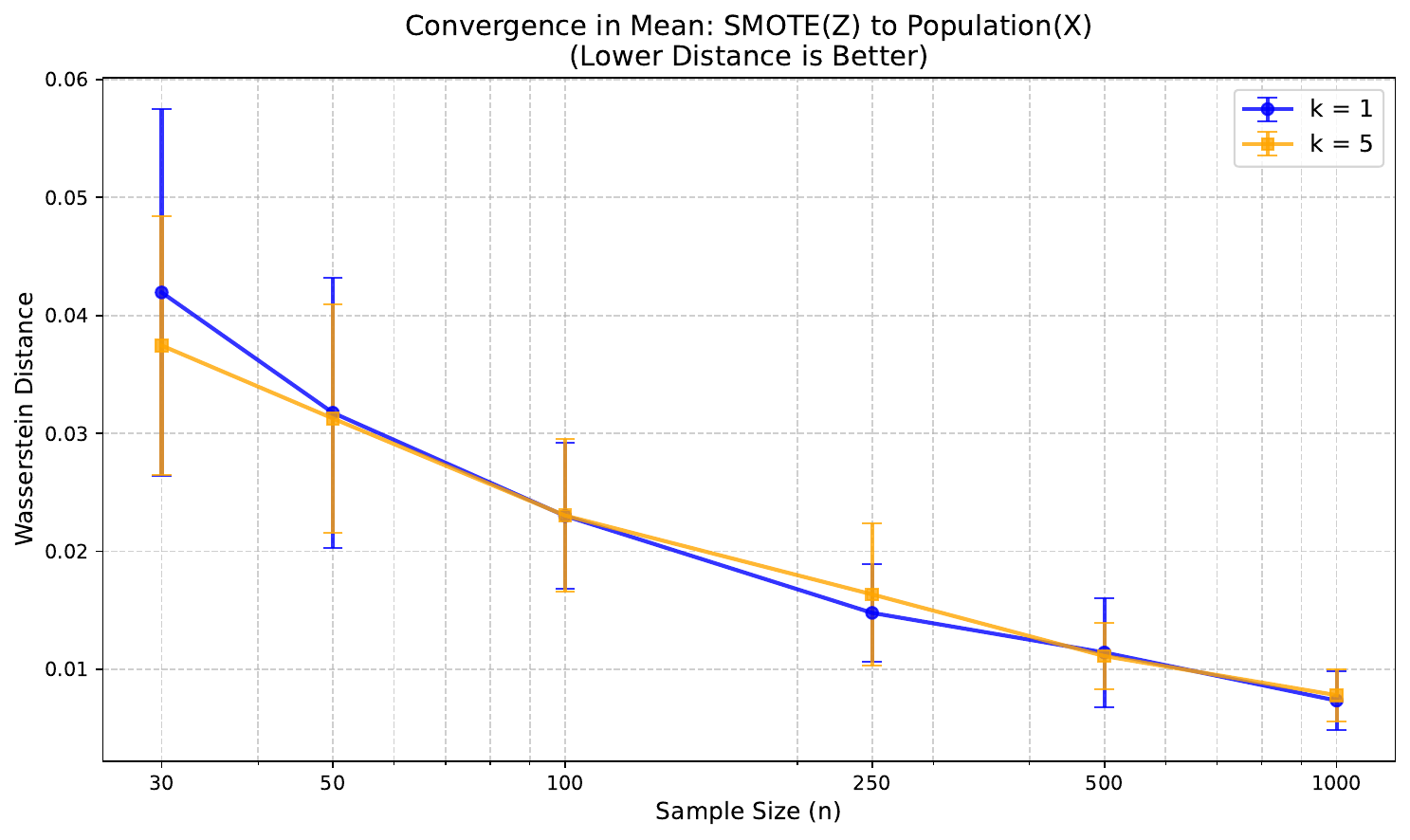}
        \caption{Convergence in mean of SMOTE-generated samples $Z$ to the true population $X$ (California Housing Median Income) as measured by the Wasserstein distance. The decreasing trend empirically validates Theorem~\ref{thm:convmean}.}
        \label{fig:exp2_cali}
\end{figure*}

\subsection{General SMOTE}
\label{sec:gensmote}
In this section, we broaden our main results related to converegence to a more general version of SMOTE.
The general SMOTE algorithm is an extension of Algorithm~\ref{alg:smote1}. The details of the general procedure are presented in Algorithm~\ref{alg:smotek}:
\begin{algorithm}
\caption{SMOTE-$K$}
\label{alg:smotek}
\begin{algorithmic}[1]
\State \textbf{Input:} A sample $(X_1, X_2, \dots, X_n)$ and rank $1\leq K\leq n-1$
\State \textbf{Output:} A synthetic sample $Z$
\State Randomly choose an instance $X_i$ from the sample
\State Find the $K$ nearest neighbors of $X_i$, denoted $(X_{i,(1)}, X_{i,(2)},... X_{i,(K)})$
\State Randomly choose a nearest neighbor $X_{i,(k)}$ from the previous step.
\State Generate a random number $\lambda \sim U(0, 1)$
\State Create a synthetic point $Z= X_i + \lambda (X_{i,(k)} - X_i)$
\State \textbf{Return} $Z$
\end{algorithmic}
\end{algorithm}

Let $Z_k$ and $Z_K$ denote the SMOTE variables generated via SMOTE-$k$ and SMOTE-$K$ procedures, respectively. Then,
$$Z_K = \frac{1}{K}\sum_{k=1}^K Z_k.$$ 
First, we extend the result about the convergence in probability from $Z_k$ to $Z_K$.

\begin{corollary}
\label{convProbK}
Let \( X \) be a continuous random variable. Let $Z$ be the random variable generated via the SMOTE-K procedure from an  i.i.d. sample  \( X_1, X_2, \dots, X_n \) drawn from \( X \). Then, $Z$ converges to $X$ in probability as $n\rightarrow \infty$.
\end{corollary}

\begin{proof}
Let $\varepsilon$ and $\delta$ be given. We know from Theorem~\ref{convProb} that for each $1\leq k \leq K$, there exists $N_k$ such that  \( P\left( \left| Z_k - X_1 \right| > \delta \right) <\varepsilon, \,\forall n > N_k\). Let $N=\sup_{N_1, N_2, ..., N_K}$. Then, by the triangle inequality

\begin{align*}
P\left( \left| Z_K - X_1 \right| > \delta \right)
&=P\left( \left| \left(\frac{1}{K}\sum_{k=1}^K Z_k\right) - X_1 \right| > \delta \right)\\
&\leq P\left( \frac{1}{K}\sum_{k=1}^K \left| Z_k - X_1 \right| > \delta \right)\\
&= \frac{1}{K}\sum_{k=1}^K P\left(  \left| Z_k - X_1 \right| > \delta \right)\\
&\leq\varepsilon,
\end{align*}
for all $n\geq N$.
\end{proof}

Next, we extend the result about the convergence in mean from $Z_k$ to $Z_K$. 
\begin{corollary}
\label{convMeanK}
Let \( X \) be a continuous random variable with compact support. Let $Z$ be the random variable generated via the SMOTE-K procedure from an  i.i.d. sample  \( X_1, X_2, \dots, X_n \) drawn from \( X \). Then, $Z$ converges to $X$ in mean as $n\rightarrow \infty$.
\end{corollary}

\begin{proof}
Let $\varepsilon$ be given. We know from Theorem~\ref{thm:convmean} that for each $1\leq k \leq K$, there exists $N_k$ such that
\(E\left[ \left|Z_k - X_{1}\right| \right] \leq \varepsilon , \,\forall n > N_k.\)
Let $N=\sup_{N_1, N_2, ..., N_K}$. Then, by the triangle inequality

\begin{align*}
E\left[ \left|Z_K - X_{1}\right| \right]
&=E\left[ \left| \left(\frac{1}{K}\sum_{k=1}^K Z_k\right) - X_{1}\right| \right]\\
&\leq \frac{1}{K}\sum_{k=1}^K E\left[ \left|Z_k - X_{1}\right| \right] \\
&\leq\varepsilon,
\end{align*}
for all $n\geq N$.
\end{proof}

\section{Optimal rank $k$}
\label{sec:num}

In the standard SMOTE algorithm, the rank $k$ of the nearest neighbor is typically set to $k=5$ \cite{Chawla, Sakho}. However, Equation~\ref{distk} demonstrates that selecting lower values of $k$ is preferable. Concretely, since $P(Y \leq k - 1 \mid X_1) \leq P(Y \leq k \mid X_1)$, it follows that
\begin{equation} 
\label{compk} 
P(D_{(k)} \geq \varepsilon \mid X_1) \leq P(D_{(k+1)} \geq \varepsilon \mid X_1). 
\end{equation} 
The above inequality implies that the expected distance between $X_1$ and its $k$-th nearest neighbor $X_{1,(k)}$ increases with $k$. Since the synthetic sample $Z$ is generated between $X_1$ and $X_{1,(k)}$, a larger $k$ results in $Z$ being farther from $X_1$. To illustrate this effect, we compare the distribution of $Z$ to that of $X$ for various values of $k$.

Figure~\ref{fig:uniform} displays the distribution of $Z$ simulated for sample sizes $n = 8, 20, 70$ based on the uniform random variable $X \sim U(0,1)$. In Figure~\ref{fig:uniform1}, we used the nearest neighbor of rank $k=1$, whereas in Figure~\ref{fig:uniform5}, we used $k=5$. As shown in the figures, for both $k=1$ and $k=5$, the distribution of $Z$ converges to that of $X$ as $n$ increases. However, the convergence is significantly faster in the case of $k=1$ compared to $k=5$, especially for lower values of $n$.

\begin{figure*}[!h]
     \centering
     \begin{subfigure}[b]{0.80\textwidth}
         \centering
         \includegraphics[width=\textwidth]{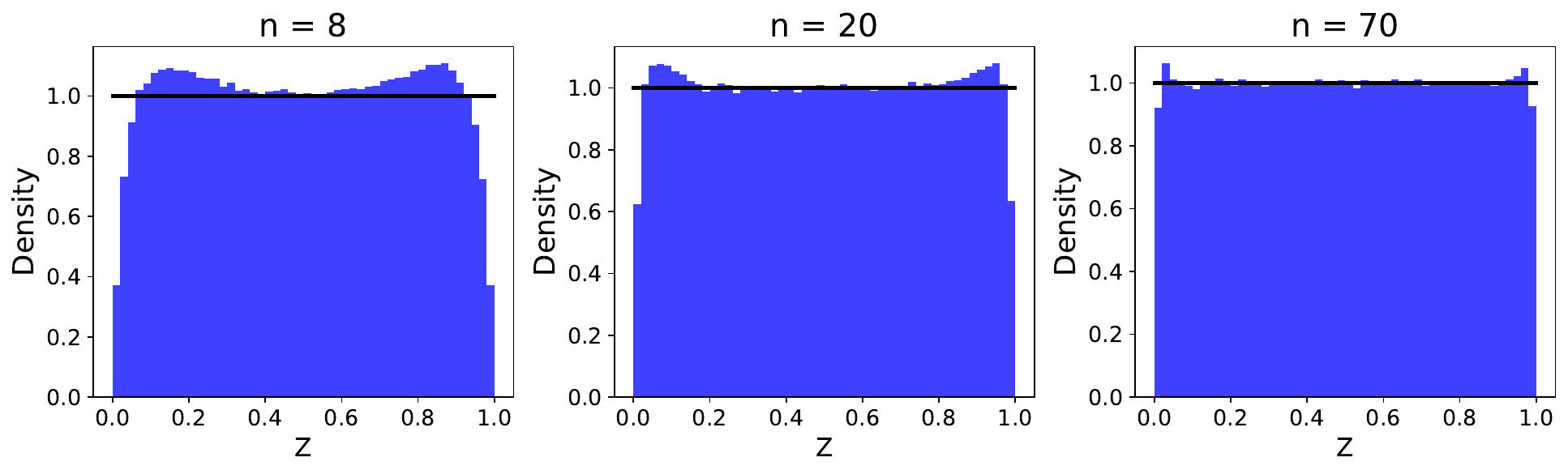}
         \caption{$k=1$}
         \label{fig:uniform1}
     \end{subfigure}
   \\
     \begin{subfigure}[b]{0.80\textwidth}
         \centering
         \includegraphics[width=\textwidth]{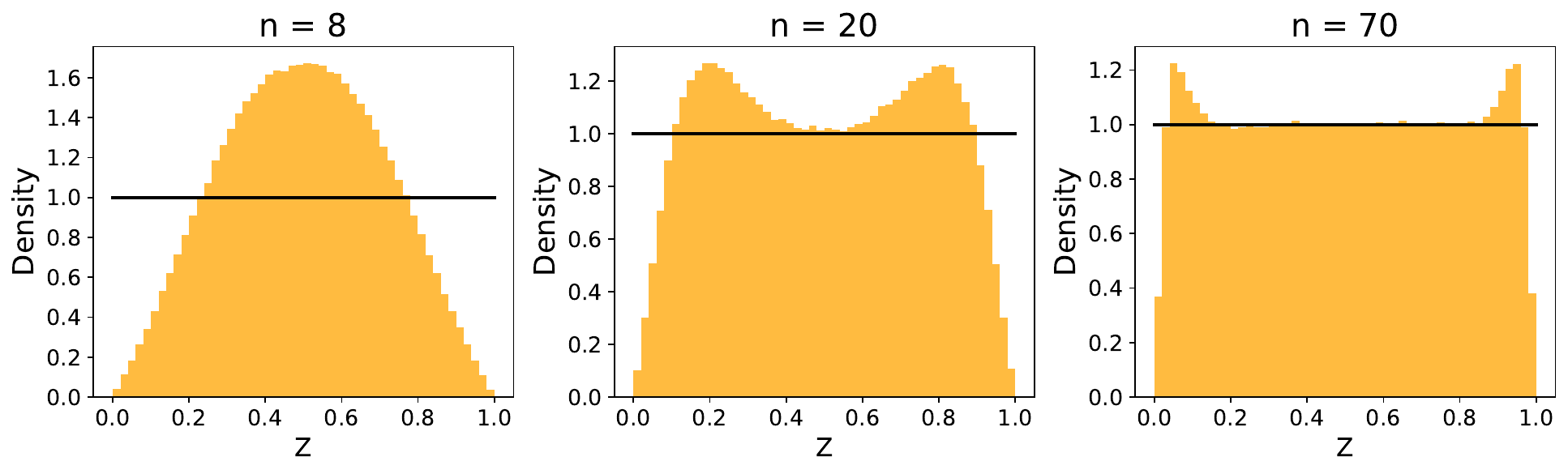}
         \caption{$k=5$}
         \label{fig:uniform5}
     \end{subfigure}
        \caption{Simulated distribution of $Z$ for sample sizes $n = 8, 20, 70$ based on  $X\sim U(0,1)$ together with the graph of $X$. It is clear that the distribution generated with $k=1$ (top) is closer to the true distribution than with $k=5$ (bottom).}
        \label{fig:uniform}
\end{figure*}

Similarly, we have conducted simulations using Gaussian ($X \sim \mathcal{N}(0,1)$) and exponential ($X \sim \mathrm{Exp}(1)$) random variables. The results, presented in Figure~\ref{fig:gaussian}, show that the distribution of $Z$ approaches to the original distribution $X$ faster for $k=1$ than for $k=5$.

\begin{figure*}[!h]
     \centering
     \begin{subfigure}[b]{0.80\textwidth}
         \centering
         \includegraphics[width=\textwidth]{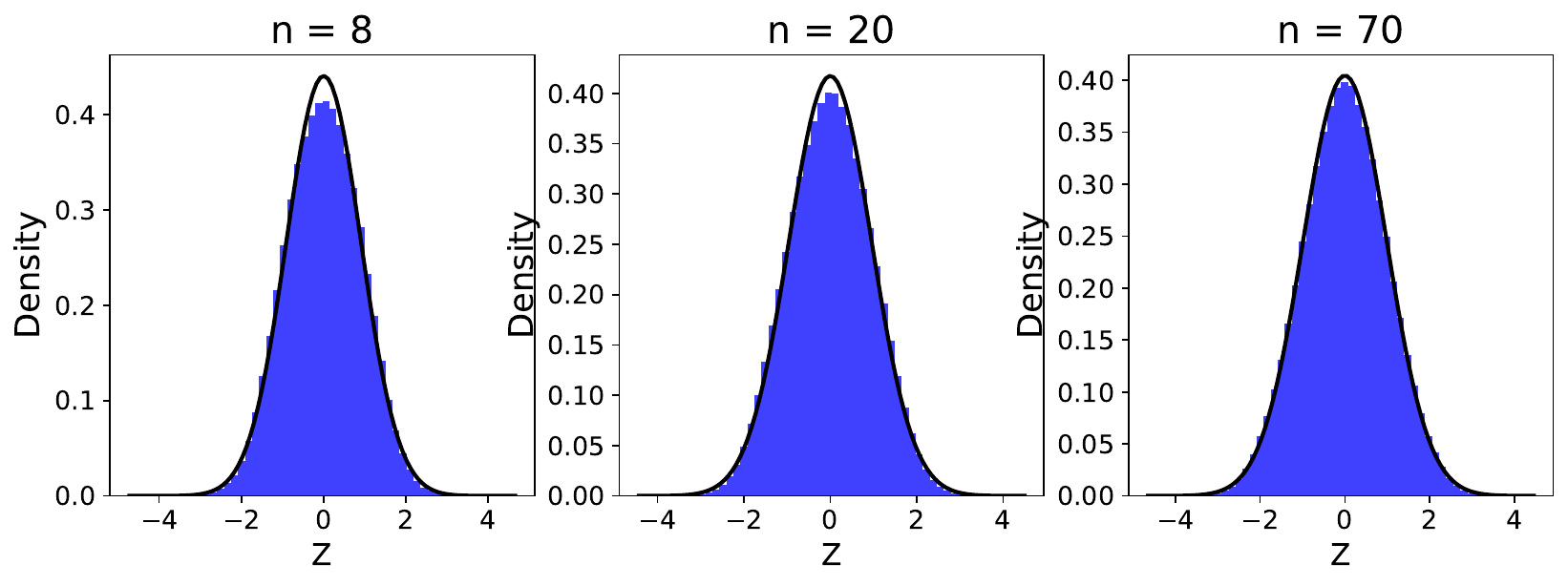}
         \caption{$k=1$}
         \label{fig:gaussian1}
     \end{subfigure}
    \\
     \begin{subfigure}[b]{0.80\textwidth}
         \centering
         \includegraphics[width=\textwidth]{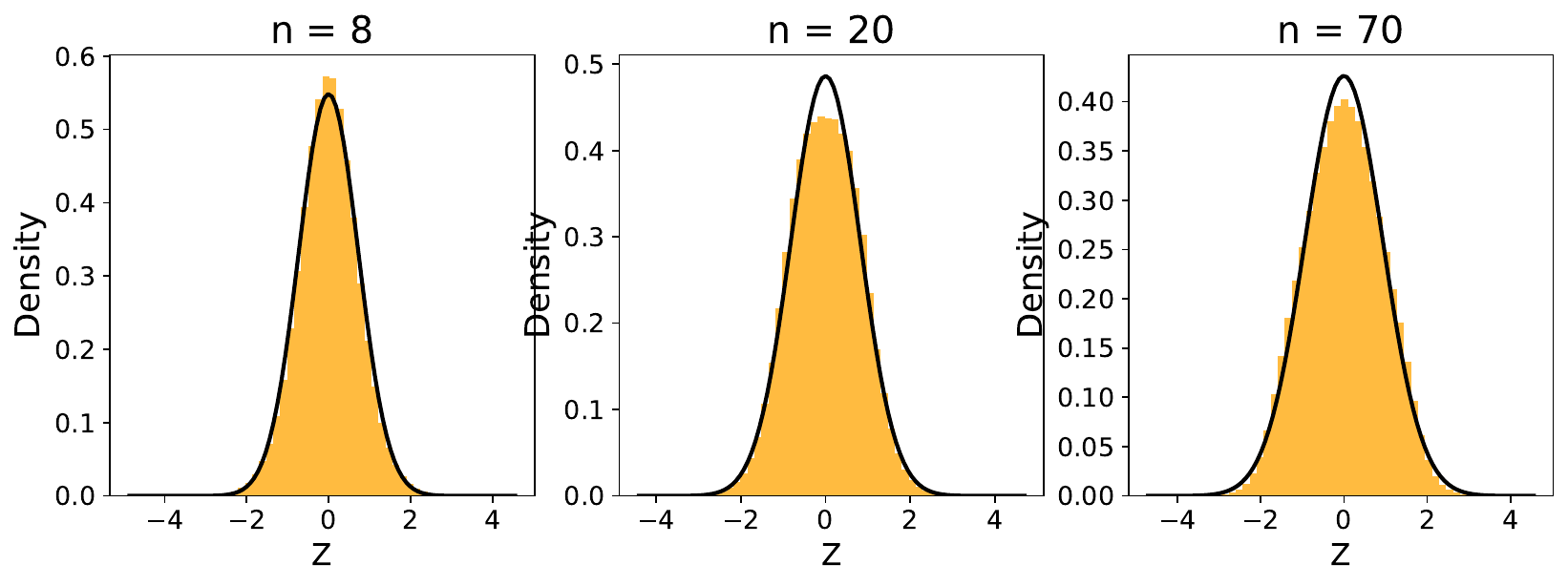}
         \caption{$k=5$}
         \label{fig:gaussian5}
     \end{subfigure}
        \caption{Simulated distribution of $Z$ for sample sizes $n=8, 20, 70$ based on $X \sim \mathcal{N}(0,1)$ together with the graph of $X$. It is clear that the distribution generated with $k=1$ is closer to the true distribution than with $k=5$.}
        \label{fig:gaussian}
\end{figure*}

\begin{figure*}[!h]
     \centering
     \begin{subfigure}[b]{0.80\textwidth}
         \centering
         \includegraphics[width=\textwidth]{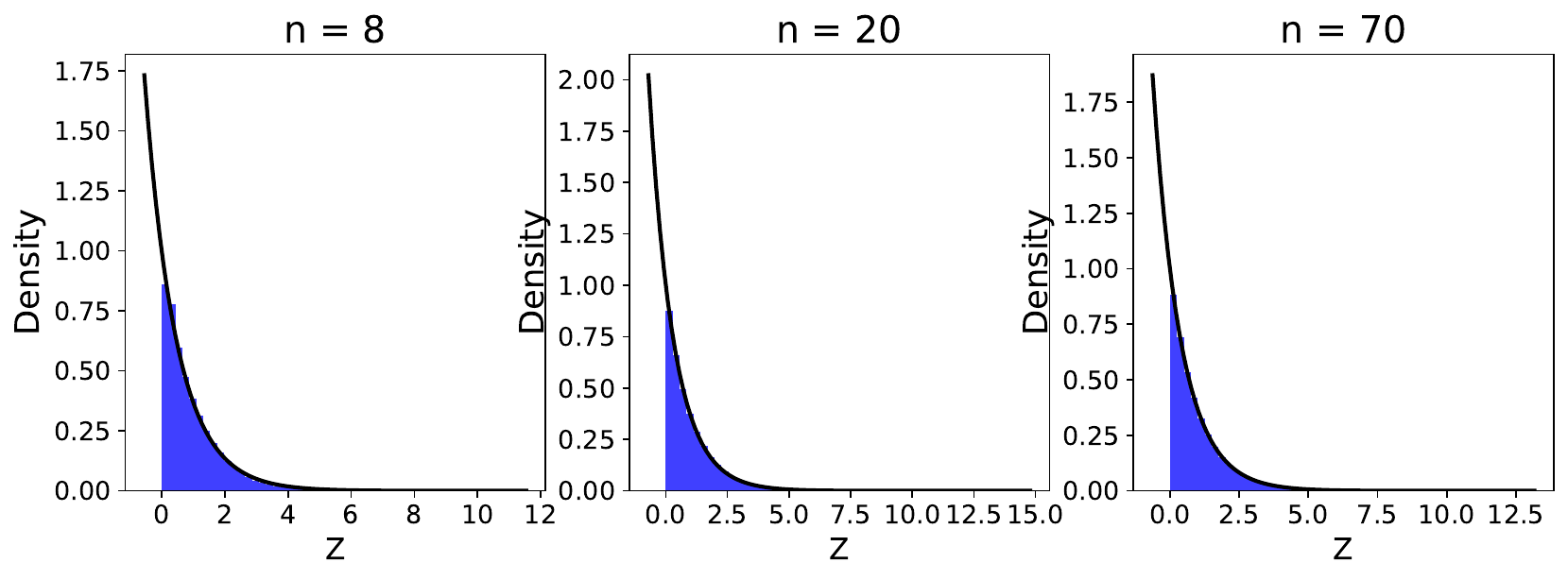}
         \caption{$k=1$}
         \label{fig:exp1}
     \end{subfigure}
    \\
     \begin{subfigure}[b]{0.80\textwidth}
         \centering
         \includegraphics[width=\textwidth]{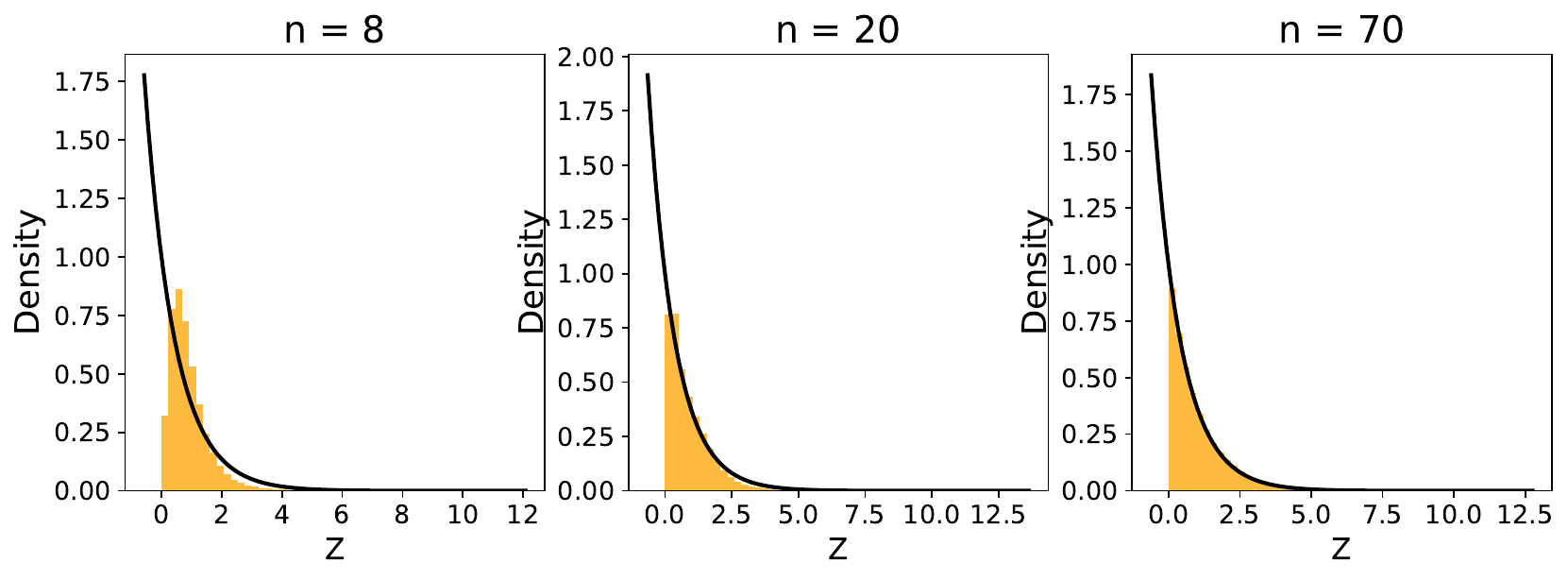}
         \caption{$k=5$}
         \label{fig:exp5}
     \end{subfigure}
        \caption{Simulated distribution of $Z$ for sample sizes $n=8, 20, 70$ based on $X \sim \mathrm{Exp}(1)$ together with the graph of $X$. It is clear that the distribution generated with $k=1$ is closer to the true distribution than with $k=5$.}
        \label{fig:exp}
\end{figure*}

To further assess the rate of convergence of the synthetic sample $Z$ to the original distribution $X$, we calculated the Kullback-Leibler (KL) divergence between the simulated distribution of $Z$ and the theoretical distribution of $X$ for nearest neighbor ranks $k = 1$ and $k = 5$ over a range of sample sizes $n$. As illustrated in Figure~\ref{fig:kl}, the KL divergence for both the uniform and Gaussian random variables decreases to zero as the sample size $n$ increases. Notably, the KL divergence is consistently smaller for $k = 1$ than for $k = 5$ across all values of $n$, particularly for smaller sample sizes. In the case of the exponential random variable, the KL divergence for $k = 1$ approaches zero, whereas for $k = 5$, it initially increases before starting to decrease around $n = 55$.

\begin{figure*}[!h]
         \centering
         \includegraphics[width=\textwidth]{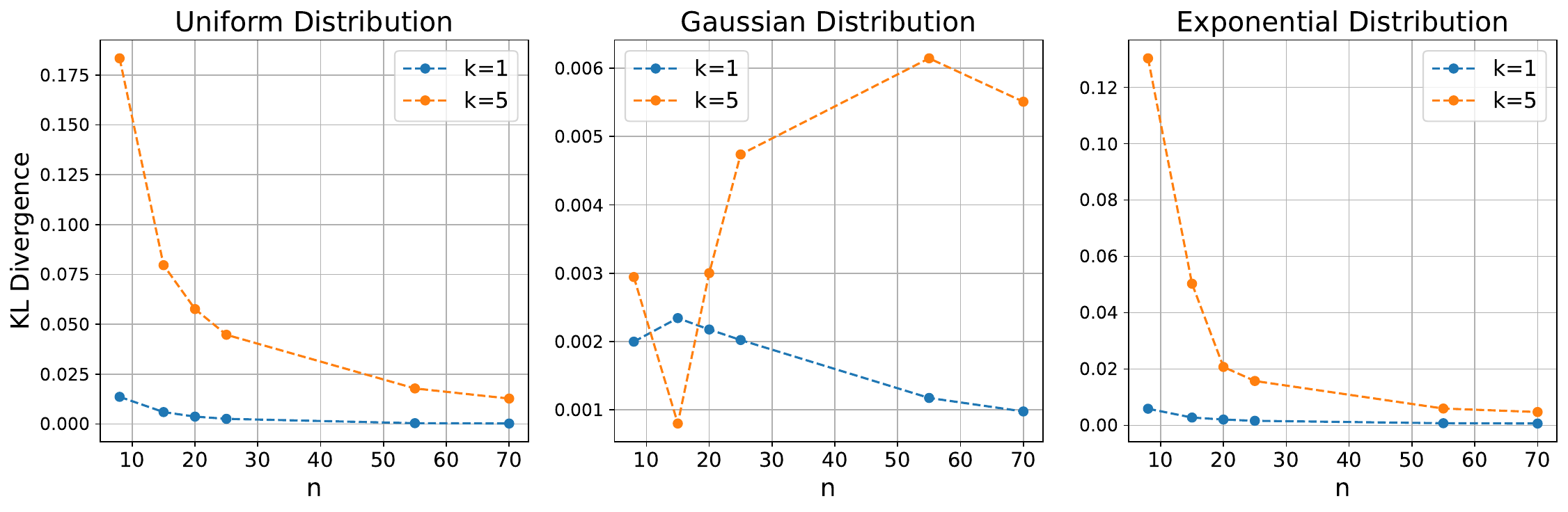}
        \caption{KL-divergence between the  simulated distribution of $Z$ and the theoretical distribution of $X$ for nearest neighbor ranks $k = 1$ and $k = 5$.}
        \label{fig:kl}
\end{figure*}

\section{Conclusion}
\label{sec:conclusion}

In this study, we conducted a thorough theoretical examination of the convergence behavior of the SMOTE algorithm. Our analysis demonstrated that the synthetic samples \( Z \) generated by SMOTE converge in probability to the original data distribution \( X \). In addition, under the assumption of compact support for \( X \), we established that \( Z \) also converges to \( X \) in mean, providing a stronger form of convergence. Thus, our results offer a theoretical validation that the SMOTE algorithm accurately approximates the original data distribution for large sample sizes.

The practical implications of our findings are significant for those applying SMOTE to address imbalanced data challenges. Our main results provide theoretical support to the validity of SMOTE, allowing researchers and practitioners to use the SMOTE algorithm with confidence. By recommending smaller \( k \) values, our work provides actionable guidance that can lead to more accurate and representative synthetic samples, thereby improving the performance and fairness of machine learning models in critical applications such as healthcare, finance, and security. Additionally, similar techniques can be applied to analyze convergence in higher dimensions. Future research could extend the theoretical framework to other data augmentation techniques, further enhancing our understanding and application of methods for handling imbalanced data.



\end{document}